\documentclass[a4paper]{article}

\usepackage{amssymb}
\usepackage{graphicx}
\usepackage{wrapfig}
\usepackage{bm}
\usepackage{url}
\usepackage{mathtools,multirow,amsthm} % define this before the line numbering.
\usepackage{algpseudocode}
\usepackage{nicefrac}
\usepackage[linesnumbered,ruled]{algorithm2e}
\usepackage{color}
\newtheorem{theorem}{Theorem}[section]
\newtheorem{corollary}{Corollary}[theorem]
\newtheorem{lemma}[theorem]{Lemma}

\DeclareMathOperator*{\argmax}{arg\,max}

\DeclareMathOperator*{\argmin}{arg\,min}

\begin{document}

%\mainmatter  % start of an individual contribution

% first the title is needed
\title{Dictionary Learning and Sparse Coding on Statistical Manifolds}

% a short form should be given in case it is too long for the running head
%\titlerunning{Dictionary Learning and Sparse Coding}

% the name(s) of the author(s) follow(s) next
%
% NB: Chinese authors should write their first names(s) in front of
% their surnames. This ensures that the names appear correctly in
% the running heads and the author index.
%

\author{
Rudrasis~Chakraborty, Monami~Banerjee,
  and~Baba~C.~Vemuri \\ Department of
  Computer and Information Science and Engineering \\ University of
  Florida,~Gainesville, FL
}
\date{}
\maketitle

\begin{abstract}
In this paper, we propose a novel information theoretic framework for
dictionary learning (DL) and sparse coding (SC) on a statistical
manifold (the manifold of probability distributions). Unlike the
traditional DL and SC framework, our new formulation does not
explicitly incorporate any sparsity inducing norm in the cost function
being optimized but yet yields sparse codes. Our algorithm
approximates the data points on the statistical manifold (which are
probability distributions) by the weighted Kullback-Leibeler
center/mean (KL-center) of the dictionary atoms. The KL-center is
defined as the minimizer of the maximum KL-divergence between itself
and members of the set whose center is being sought. Further, we prove
that the weighted KL-center is a sparse combination of the dictionary
atoms. This result also holds for the case when the KL-divergence is
replaced by the well known Hellinger distance.
%% Since,
%% the data reside on a statistical manifold, the data fidelity term can
%% not be the same as in the case of the vector-space data. Therefore,
%% we employ the geodesic distance between the data and a sparse
%% approximation of the data element. This cost function is minimized
%% using an accelerated gradient descent algorithm. 
From an applications perspective, we present an extension of the
aforementioned framework to the manifold of symmetric positive
definite matrices (which can be identified with the manifold of zero
mean gaussian distributions), $\mathcal{P}_n$.  We present experiments
involving a variety of dictionary-based reconstruction and
classification problems in Computer Vision. Performance of the
proposed algorithm is demonstrated by comparing it to several
state-of-the-art methods in terms of reconstruction and
classification accuracy as well as sparsity of the chosen
representation.

\end{abstract}
\section{Introduction}\label{sec:introduction}
% Computer Society journal (but not conference!) papers do something unusual
% with the very first section heading (almost always called "Introduction").
% They place it ABOVE the main text! IEEEtran.cls does not automatically do
% this for you, but you can achieve this effect with the provided
% \IEEEraisesectionheading{} command. Note the need to keep any \label that
% is to refer to the section immediately after \section in the above as
% \IEEEraisesectionheading puts \section within a raised box.

Dictionary learning and sparse coding have found wide applicability in
Image/Signal processing, Machine Learning and Computer Vision in
recent times. Examples applications include but are not limited to,
image classification
\cite{Mairal2009,qiu2014information}, image restoration
\cite{Wright2010} and face recognition \cite{Wright2009} and many
others.  The traditional dictionary learning (DL) and sparse coding
(SC) formulation assumes that the input data lie in a vector space,
and assumes a linear generative model for the data by approximating
the data with a sparse linear combination of the dictionary atoms
(elements).  Thus, the objective function of the DL problem typically
has a data fidelity term to minimize the ``reconstruction error'' in
the least squares sense. Sparsity is then enforced on the weights in
the linear combination via a tolerance threshold on the $\ell_0$-norm
of the weight vector. This however leads to an NP-hard problem and the
most popular approach for solving this problem (with no convergence
guarantees) is the K-SVD based approach \cite{aharon2005k}. For a
fixed dictionary, a convex approximation to the $\ell_0$-norm
minimization to induce sparsity can be achieved using the
$\ell_1$-norm constraint on the weight vector
\cite{candes2006stable,donoho2012sparse,akhtar2016discriminative}. The
problem of finding both the optimal dictionary and the sparse-codes
however remains to be a hard computational problem in general. For
further discussion on this topic and the problem of complete
dictionary recovery over a sphere, we refer the readers to
\cite{Sun2015}, where authors provide a provably convergent algorithm.

In many application domains however, the data do not reside in a
vector space, instead they reside on a Riemannian manifold such as the
Grassmannian \cite{chakraborty2015iccv,cetingul2009intrinsic}, the
hypersphere
\cite{MardiaBook,srivastava2007riemannian,salehian2015efficient}, the
manifold of symmetric positive definite (SPD) matrices
\cite{Moakher_simax05,LengletRDF_jmiv06,fletcher2007riemannian,sra2011generalized,Xie2013}
and many others. Generalizing the DL \& SC problem from the case of
vector space inputs to the case when the input data reside on a
Riemannian manifold is however difficult because of the nonlinear
structure of Riemannian manifolds \cite{Xie2013}. One could consider
embedding the Riemannian manifold into a Euclidean space, but a
problem with this method is that there does not exist a canonical
embedding for a general Riemannian manifold.  This motivated
researchers
\cite{Xie2013,Harandi2012,sra2011generalized,Li2013,zhang2017analytic}
to generalize the DL and SC problem to Riemannian manifolds. Though
the formulation on a Riemannian manifold involves a ``reconstruction
error'' term analogous to the vector space case, defining a sparsity
inducing constraint on a manifold is nontrivial and should be done
with caution. This is because, a Riemannian manifold lacks ``global''
vector space structure since it does not have the concept of a global
origin. Hence, as argued in \cite{Xie2013}, one way to impose the
sparsity inducing constraint is via an an affine constraint, i.e., the
sparsity constraint is over an affine subspace defined by the tangent
space at each data point on the manifold. We now briefly review a few
representative algorithms for the DL \& SC problem on Riemannian
manifolds.

A popular solution to the DL problem is to make use of the tangent
spaces, which are linear spaces associated with each point on a
Riemannian manifold. This approach essentially involves use of linear
approximation in the smooth neighborhood of a point.  Guo et
al. \cite{Guo2013} use a Log-Euclidean framework described at length
in \cite{arsigny2007geometric} to achieve a sparse linear
representation in the tangent space at the Fr\'{e}chet mean of the
data. Xie et al. \cite{Xie2013} developed a general dictionary
learning formulation that can be used for data on any Riemannian
manifold. In their approach, for the SC problem, authors use the
Riemannian Exponential (Exp.) and Logarithm (Log.) maps to define a
generative process for each data point involving a sparse combination
of the Log.-mapped dictionary atoms residing on the manifold. This
sparse combination is then realized on the manifold via the Exp.-map.
Their formulation is a direct generalization of the linear sparsity
condition with the exception of the origin of the linear space being
at the data point. Further, they impose an affine constraint in the
form of the weights in the weight vector summing to one. This
constraint implies the use of affine subspaces to approximate the
data. For fixed weights however, estimating the dictionary atoms is a
hard problem and a manifold line search method is used in their
approach. In another method involving DL and SC on the manifold of SPD
matrices, Cherian et al. \cite{Cherian2014} proposed an efficient
optimization technique to compute the sparse codes. Most recently,
authors in \cite{schmitz2017wasserstein} introduced a novel nonlinear DL
and SC method for histograms residing in a simplex. They use the well
known Wasserstein distance along with an entropy regularization
\cite{Cuturi} to reconstruct the histograms that are Wasserstein
barycenter approximations of the given data (histograms). They solve
the resulting optimization for both the dictionary atoms and weights
using a gradient based technique. Authors point out that using the
entropy regularization leads to a convex optimization
problem. However, authors did not discuss sparsity of the ensuing
Wasserstein barycenter dictionary based representation. Sparsity
property is of significant importance in many applications and the
focus of our work here is on how to achieve sparsity without
explicitly enforcing sparsity inducing constraints.

Several recent works report the use of kernels to accomplish
dictionary learning and sparse coding on Riemannian manifolds
\cite{Harandi2015,Li2013,Harandi2012}.  In these, the Riemannian
manifold is embedded into the Reproducing Kernel Hilbert Space
(RKHS). DL and SC problems are then formulated in the RKHS. RKHS is a
linear space, and hence it is easier to derive simple and effective
solutions for the DL and SC problems.  Recently, authors in
\cite{feragen2015geodesic} presented conditions that must be strictly
satisfied by geodesic exponential Kernels on general Riemannian
manifolds.  This important and significant result %% exposed the weakness
%% of some of the earlier kernel-based approaches in literature and
%% considered in \cite{Harandi2015,Li2013,Harandi2012} do not
%% lead to positive definite Kernels as claimed by the authors and thus
provides guidelines for designing a kernel based approach for general
Riemannian manifolds.

%%% We need to include more recent literature on kernel-based methods
%%% published in PAMI and other journals recently. Chellappa's group
%%% has a paper in PAMI and some Chinese authors have a paper as
%%% well. Please do this.

In this work, we present a novel formulation of the DL and SC problems
for data residing on a statistical manifold, without explicitly
enforcing a sparsity inducing constraint. The proposed formulation
circumvents the difficulty of directly defining a sparsity constraint
on a Riemannian manifold. Our formulation is based on an information
theoretic framework and is shown to yield sparse codes. Further, we
extend this framework to the manifold of SPD matrices. Note that SPD
matrices can be identified with the space of zero mean Gaussian
distributions, which is a statistical manifold.  Several experimental
results are presented that demonstrate the competitive performance of
our proposed algorithm in comparison to the state-of-the-art. 

The rest of the paper is organized as follows: in Section \ref{sec3},
we first present the conventional DL and SC problem formulation in
vector spaces and motivate the need for a new formulation of the DL
and SC problem on Riemannian manifolds. This is followed by a brief
summary of relevant mathematical background on statistical
manifolds. Following this, we summarize the mathematical results in
this paper and then present the details along with our algorithm for
the DL and SC problem. In Section \ref{sec4}, we present several
experimental results and comparisons to the state-of-the-art. Finally,
in Section \ref{sec5}, we draw conclusions.

\section{An Information Theoretic Formulation}\label{sec3}
In the traditional SC problem, a set of data vectors 
$X=\{\mathbf{x}_i\}_{i=1}^N \subset \mathbf{R}^n$, and a collection of
atoms, $A=\{\mathbf{a}_j\}_{j=1}^r \subset \mathbf{R}^n$, are
given. The goal is to express each $\mathbf{x}_i$ as a sparse linear
combination of atoms in $A$. Let, $A$ be the (overcomplete) dictionary
matrix of size $n \times r$ whose $i^{th}$ column consists of
$\mathbf{a}_i$. Let $W =[\mathbf{w}_1, \cdots, \mathbf{w}_N]$ be a $r
\times N$ matrix where each $\mathbf{w}_i \in \mathbf{R}^r$ consists
of the coefficients of the sparse linear combination. In the DL and SC
problem, the goal is to minimize the following objective function:
\begin{equation}
\min_{A, w_1, \cdots, w_n} \sum_{i=1}^n \|x_i - A w_i\|^2 + \mathbf
{Sp}(w_i), \label{EQ:One}
\end{equation}
%\begin{equation}
%\min_{\mathbf{w}_i \in \mathbf{R}^r} \|\mathbf{w}_i\|_0\:\: \text{\bf s.t.   } \mathbf{x}_i \approx A\mathbf{w}_i, \forall i 
%\end{equation} 
Here, $\mathbf{Sp}(w_i)$ denotes the sparsity promoting term, which
can be either an $\ell_0$ norm or an $\ell_1$ norm.  Since, in the
above optimization problem, both the dictionary $A$ and the
coefficient matrix $W$ are unknown, it leads to a hard optimization
problem. As this optimization problem is computationally intractable
when the sparsity promoting term is an $\ell_0$ norm constraint, most
existing approaches use a convex relaxation to this objective using an
$\ell_1$ norm in place of the $\ell_0$ norm constraint when performing
the sparse coding.  Now, instead of the traditional DL \& SC setup
where data as well as atoms are vector valued, we address the problem
when each data point and the atom are probability densities, which are
elements of a statistical manifold (see formal definition below). In
this paper, we present a novel DL and SC framework for data residing
on a statistical manifold. Before delving into the details, we will
briefly introduce some pertinent mathematical concepts from
Differential Geometry and Statistical Manifolds and refer the reader
to \cite{do1992riemannian,amari1987differential} for details.

\subsection{Statistical Manifolds: Mathematical Preliminaries} \label{sec21}
Let $\mathcal{M}$ be a smooth ($C^{\infty}$) manifold
\cite{do1992riemannian}. We say that $\mathcal{M}$ is $n$-dimensional
if $\mathcal{M}$ is locally Euclidean of dimension $n$, i.e., locally
diffeomorphic to $\mathbf{R}^n$. Equipped with the {\emph Levi-Civita
  connection} $\nabla$, the triplet $(M, g, \nabla)$ is called a {\it
  statistical manifold} whenever both $\nabla$ and the dual
connection, $\nabla^*$ are torsion free
\cite{do1992riemannian,amari1987differential}.

A point on an $n$-dimensional {\it statistical manifold},
$\mathfrak{D}$ (from here on, we will use the symbol $\mathfrak{D}$ to
denote a {\it statistical manifold} unless specifically mentioned
otherwise), can be identified with a (smooth) probability distribution
function on a measurable topological space $\Omega$, denoted by
$P(\mathbf{x};\bm{\theta})$
\cite{suzuki2014information,amari1987differential}. Here, each
distribution function can be parametrized using $n$ real variables
$(\theta_1, \cdots, \theta_n)$. So, an open subset $S$ of a {\it
  statistical manifold}, $\mathfrak{D}$, is a collection of the
probability distribution functions on $\Omega$. And the chart map is
the mapping from $S$ to the {\it parameter space}, $\Theta =
\{\bm{\theta}\} \subset \mathbf{R}^n$. Let $\mu$ be a $\sigma$-finite
additive measure defined on a $\sigma$-algebra of subsets of
$\Omega$. Let $f(\mathbf{x};\bm{\theta})$ be the density of
$P(\mathbf{x};\bm{\theta})$ with respect to the measure $\mu$ and
assume the densities to be smooth $(C^{\infty})$ functions. Now, after
giving $\mathfrak{D}$ a topological structure, we can define a
Riemannian metric as follows. Let $l(\mathbf{x};\bm{\theta}) = \log
f(\mathbf{x};\bm{\theta})$, then a Riemannian metric, $g$ can be
defined as $g_{ij}(\bm{\theta}) = E_{\bm{\theta}} \left[
  \frac{\partial l(\mathbf{x};\bm{\theta})}{\partial \theta_i}
  \frac{\partial l(\mathbf{x};\bm{\theta})}{\partial
    \theta_j}\right]$, where $E_{\bm{\theta}}[\mathbf{y}]$ is the
expectation of $\mathbf{y}$ with respect to $\bm{\theta}$. In general,
$g=[g_{ij}]$ is symmetric and positive semi-definite. We can make $g$
positive definite by assuming the functions $\{\frac{\partial
  l(\mathbf{x};\bm{\theta})}{\partial \theta_i}\}_{i=1}^n$ to be
linearly independent. This metric is called the {\it Fisher-Rao}
metric \cite{rao2009fisher,shun2012differential} on $\mathfrak{D}$.

\subsection{Summary of the mathematical results}

In the next section, we propose an alternative formulation to the DL
and SC problem.  We first state a few theorems as background material
that will be used subsequently. Then, we define the new objective
function for the DL and SC problem posed on a statistical manifold in
Section \ref{sec111}. Our key mathematical results are stated in
Theorems \ref{thm4} and \ref{thm5}, Corollary \ref{cor1} and
\ref{cor2} respectively. Using these results, we show that our DL \&
SC framework, {\it which does not have an explicit sparsity
  constraint}, yields sparse codes. Then, we extend our DL and SC
framework to the manifold of SPD matrices, $\mathcal{P}_n$, in Section
\ref{sec112}.

\subsection{Detailed mathematical results}
%% In this section, we will first formulate the DL and SC problem for
%% data residing on a {\it statistical manifold}. Then, we extend this
%% framework to the manifold of SPD matrices. 
Let the manifold of probability densities, hereafter denoted by
$\mathfrak{D}$ be the $n$-dimensional statistical manifold, i.e., each
point on $\mathfrak{D}$ is a probability density.  We will use the
following notations in the rest of the paper.
\begin{itemize}
 \item Let $\mathfrak{G}$ be a dictionary with $r$ atoms $g_1$,
   $\cdots,$ $g_r$, where each $g_i \in \mathfrak{D}$.
 \item Let $\mathfrak{F} = \{f_i\}_{i=1}^N$ $\subset \mathfrak{D}$ be
   a set of data points.
 \item And $w_{ij}$ be nonnegative weights corresponding to $i^{th}$
   data point and $j^{th}$ atom, $i \in \{1,\cdots, N\}$ and $j \in
   \{1, \cdots, r\}$.
\end{itemize}
Note that, here we assume that each density $f$ or $g$ is
parameterized by $\bm{\theta}$.  There are many ways to measure the
discrepancy between probability densities. One can choose an intrinsic
metric and the corresponding distance on a statistical manifold to
measure this discrepancy, such as the Fisher-Rao metric
\cite{rao2009fisher,shun2012differential}, which however is expensive to
compute. In this paper, we choose an extrinsic measure namely the
non-negative divergence measure called the Kullback-Leibler (KL)
divergence.  The KL divergence \cite{cover2012elements} between two
densities $f_1$ and $f_2$ on $\mathfrak{D}$ is defined by
\begin{equation}
\text{KL}(f_1, f_2) = \int f_1(x) \log \frac{f_1(x)}{f_2(x)} dx
\end{equation}
The Hessian of KL-divergence is the Fisher-Rao metric defined
earlier. In other words, the KL-divergence between two nearby
probability densities can be approximated by half of the squared
geodesic distance (induced by the Fisher-Rao metric) between them
\cite{shun2012differential}. {\it The KL-divergence is not a distance
  as it is not symmetric and does not satisfy the triangle
  inequality}. It is a special case of a broader class of divergences
called the $f$-divergences as well as of the well known class of
Bregman-divergences. We refer the reader to
\cite{basu1998robust,liese2006divergences} for more details in this
context.  Given a set of densities $\mathfrak{F} = \{f_i\}$, the KL
divergence from $\mathfrak{F}$ to a density $f$ can be defined by,
\begin{equation}
\label{neq1}
\text{KL}(\mathfrak{F}, f) = \displaystyle \max_i \text{KL}(f_i, f).
\end{equation}
%The KL divergence is a induced divergence measure on $\mathfrak{D}$ by the following Theorem.
%\begin{theorem}
%\label{thm1}
%The Hessian of the KL divergence is the {\it Fisher-Rao} metric.
%\end{theorem}
We can define the {\it KL-center} of $\mathfrak{F}$, denoted by $f_m(\mathfrak{F})$, by
\begin{equation}
\label{eq2}
f_m(\mathfrak{F}) = \argmin_f \text{KL}(\mathfrak{F},f).
\end{equation}
The symmetrized KL divergence, also called the Jensen-Shannon
divergence (JSD) \cite{cover2012elements} between two densities $f_1$
and $f_2$ is defined by,
\begin{equation}
\label{eq11}
\text{JSD}(f_1, f_2) = \frac{1}{2} \text{KL}(f_1, f_2) + \frac{1}{2} \text{KL}(f_2, f_1).
\end{equation}
In general, given the set $\mathfrak{F}=\{f_i\}$, define a mixture of
densities as, $f = \sum_i \alpha_i f_i$, $\sum \alpha_i =1$, $\alpha_i
\geq 0$, $ \forall i$. It is evident that the set of $\{\alpha_i\}$
forms a simplex, which is denoted here by $\Delta$. Then, the JSD of
the set $\mathfrak{F}$ with the mixture weights $\{\alpha_i\}$ is
defined as,
\begin{equation}
\label{eq3}
\text{JSD}(\{f_i\}) = H(\sum_i \alpha_i f_i) - \sum_i \alpha_i H(f_i),
\end{equation}
where $H(f) = -\int f(x) \log f(x) dx$ is the Shannon entropy of the
density $f$. It is easy to see the following Lemma.
\begin{lemma}
$\text{JSD}(\{f_i\})$ is concave in $\{\alpha_i\}$ and JSD attains the
  minimum at an extreme point of the simplex $\Delta$.
\end{lemma}
\begin{proof}
We refer the reader to \cite{ericthesis} for a proof of this Lemma. 
\end{proof}
In \cite{ericthesis}, it was shown that one can compute the {\it KL-
  center} of $\mathfrak{F}$, $f_m(\mathfrak{F})$, in Equation
\ref{eq2} using the following theorem:
\begin{theorem}
\label{thm2}
The {\it KL center} of $\mathfrak{F}$, $f_m(\mathfrak{F})$, is
given by
\begin{eqnarray*}
f_m(\mathfrak{F}) &=& \sum_i \hat{\alpha}_i f_i \\
\text{where  } \hat{\bm{\alpha}} &=& \argmax_{\bm{\alpha}} \text{JSD}(\{f_i\}).  
\end{eqnarray*}
\end{theorem}
\begin{proof}
We refer the reader to \cite{ericthesis} for a proof of this theorem. 
\end{proof}

Observe that, the $\text{KL}(\mathfrak{F}, f)$ defined in
Eq. \ref{neq1} has the positive-definiteness property, i.e.,
$\text{KL}(\mathfrak{F}, f) \geq 0$ for any $\mathfrak{F}$ and $f$ and
$\text{KL}(\mathfrak{F}, f) = 0$ if and only if $f \in
\mathfrak{F}$. Both of these properties are evident from the
definition of the KL divergence between two densities.

{\bf Coding theory interpretation:} It should be noted that the above
result is the same as the well known redundance-capacity theorem of
coding theory presented in
\cite{gallager1968information,davisson1980source,ryabko1994fast,akhtar2016discriminative}. The
theorem establishes the equivalence of: the minmax excess risk i.e.,
redundancy for estimating a parameter $\theta$ from a family/class
$\Theta$ of sources, the Bayes risk associated with the least
favorable prior and the channel capacity when viewing statistical
modeling as a communication via a noisy channel. In
\cite{akhtar2016discriminative}, a stronger result was shown, namely
that, the capacity is also a lower bound for ``most'' sources in the
class. The results in \cite{ericthesis} however approached this
problem from a geometric viewpoint i.e., one of finding barycenters of
probability distributions using the KL-divergence as the ``distance''
measure. Our work presented here takes a similar geometric view point
to the problem at hand, namely the DL-SC problem.

Moving on, we now define the $\ell_p$ KL divergence, denoted by
$\text{KL}_p(\mathfrak{F}, f)$, $p > 0$ as:
\begin{equation}
\label{neq2}
\text{KL}_p(\mathfrak{F}, f) = \|\left(\text{KL}(f_1, f), \cdots, \text{KL}(f_N, f)\right)^t\|_p.
\end{equation} 
where, $\|.\|_p$ is the $\ell_p$ norm of the vector and $\mathfrak{F}
= \{f_i\}_{i=1}^N$. It is easy to prove the following property of
$\text{KL}_p(\mathfrak{F}, f)$ in the following Lemma.
\begin{lemma}
$\text{KL}_p(\mathfrak{F}, f)$ as defined in Eq. \ref{neq2} is a
  well-defined {\it statistical divergence} for any $p >
  0$. Furthermore, the KL divergence as defined in Eq. \ref{neq1} is a
  special case of $\text{KL}_p$ when $p =\infty$.
\end{lemma}
Without any loss of generality, we will assume $p=1$ and refer
$\ell_1$ KL-center to be simply the KL-center for the rest of the paper
(unless mentioned otherwise). Now, given the set of densities
$\mathfrak{F} = \{f_i\}_{i=1}^N$ and a set of weights
$\{\alpha_i\}_{i=1}^N$, we can define the {\it weighted KL-center},
denoted by $f_m(\mathfrak{F}, \{\alpha_i\})$ as follows:
\begin{equation}
\label{eq4}
f_m(\mathfrak{F}, \{\alpha_i\}) = \argmin_f \sum_i\alpha_i\text{KL}(f_i, f).
\end{equation}
We like to point out that it is however easy to see that the
$\ell_{\infty}$ KL-center can not be generalized to the corresponding
weighted KL center. The above defined weighted KL-center has the
following nice property:
\begin{lemma}
The weighted KL-center as defined in Eq. \ref{eq4} is a generalization
of the KL-center in Eq. \ref{eq2} (with $p = 1$). The KL-center can be
obtained from the weighted KL-center by substituting $\alpha_i = 1/N$,
for all $i$.
\end{lemma}
\begin{theorem}
\label{thm3}
Given $\mathfrak{F}$ and $\{\alpha_i\}$ as above, $f_m(\mathfrak{F}, \{\alpha_i\}) = \sum_i \alpha_i f_i$
\end{theorem}
\begin{proof}
For simplicity, assume that each $f_i$ is discrete and assume that it
can take on $k$ discrete values, $x_1, \cdots, x_k$. Then, consider
the minimization of $\sum_i\text{KL}(f_i, f)$ with respect to $f$
subject to the constraint that $f$ is a density, i.e., for the
discrete case, $\sum_j f(x_j)=1$. By using a Lagrange multiplier
$\lambda$, we get,
\begin{eqnarray*}
\frac{\partial}{\partial f(x_j)} \left\{\sum_i\alpha_i\text{KL}(f_i, f) + \lambda \left(\sum_j f(x_j) - 1\right)\right\} &=& 0 , \:\forall j\\
\implies \left(\lambda - \frac{\sum_i \alpha_i f_i(x_j)}{f(x_j)}\right) &=& 0, \: \forall j \\
\implies f(x_j) = \frac{\sum_i \alpha_i f_i(x_j)}{\lambda}, \: \forall j.
\end{eqnarray*}
Now, by taking $\frac{\partial}{\partial \lambda}
\left\{\sum_i\alpha_i\text{KL}(f_i, f) + \lambda \left(\sum_j f(x_j) -
1\right)\right\}$ and equating to $0$, we get $f(x_j) = \sum_i
\alpha_i f_i(x_j)$, $\forall j$. Thus, $f = \sum_i \alpha_i f_i$. We
can easily extend this to the case of continuous $f_i$ by replacing
summation with integration and obtain a similar result.
\end{proof}
\subsubsection{DL and SC on a {\it statistical manifold}}
\label{sec111}
Now, we will formulate the DL and SC problems on a {\it statistical
  manifold}. The idea is to express each data point, $f_i$ as a sparse
weighted combination of the dictionary atoms, $\{g_j\}$. Given the
above hypothesis, our objective function is given by:
\begin{align}
\label{eq5}
\displaystyle \argmin_{\mathfrak{G}^*, W^*} \:\:\:\:E &  = \sum_{i=1}^N \text{KL}\left(f_i, \hat{f}_i\right) \\
\text{subject to} & \: \: \: w_{ij} \geq 0, \forall i,j \\
& \sum_j w_{ij} = 1, \forall i.
\end{align}
where, $\hat{f}_i = \sum_{j=1}^r w_{ij}g_j$, $\forall i$. In the above
objective function, $\hat{f}_i$ is the minimizer of the {\it weighted
  KL-center} of $\{g_j\}_{j=1}^r$ with weights $\{w_{ij}\}_{j=1}^r$.
The constraint $w_{ij} \geq 0$ and $\sum_j w_{ij} = 1$ is required to
make $\hat{f}_i$ a probability density. Note that, we can view
$\hat{f}_i$ as a reconstructed density from the dictionary elements
$\{g_j\}$ and weights $\{w_{ij}\}$. {\it We will now prove one of our key
  result} namely, that the minimization of the above objective
function with respect to $\{w_{ij}\}$ yields a sparse set of weights.
\begin{theorem}
\label{thm4}
Let $\mathfrak{G} = \{g_j\}$ and $W = [w_{ij}]$ be the solution of the
objective function $E$ in Equation \ref{eq5}. Then,
\begin{align*}
\left(\forall j \right), \text{KL}(f_i, g_j) \geq r_i , \text{ where } r_i = \sum_k w_{ik} \text{KL}(f_i, g_j)
\end{align*}
\end{theorem}
\begin{proof}
Consider the random variables $X_1, \cdots, X_N$ with the respective
densities $f_1, \cdots, f_N$. Since each dictionary element
$g_j$ is ``derived'' from $\{f_i\}$, hence, we can view each $g_j$ to
be associated with a random variable $Y_j$ such that $Y_j =
\tilde{g}_j\left(\{X_i\}\right)$, i.e., $Y_j$ is a transformation of
random variables $\{X_i\}$. We now have,
\begin{eqnarray*}
\begin{split}
E &= \sum_{i=1}^N \text{KL}\left(f_i, \sum_{j=1}^r w_{ij}g_j\right) \\
&= \sum_{i=1}^N \left[\int \left\{f_i(x) \log(f_i(x)) dx\right\} \right. - \\ 
& \left. \int \left\{f_i(x) \log\left( \sum_j w_{ij} g_j(x)\right)\right\} dx\right] 
\end{split}
\end{eqnarray*}
Using Jensen's inequality we have,
\begin{eqnarray*}
\begin{split}
E&\leq& \sum_{i=1}^N \left[\int \left\{f_i(x) \log(f_i(x)) dx\right\} - \right. \\
&& \left. \int \left\{f_i(x) \sum_j w_{ij} \log(g_j(x))\right\} dx\right] \\
&=& \sum_{i=1}^N E_{X_i}\left[ \log(f_i) - \sum_j w_{ij} \log(g_j)\right]
\end{split}
\end{eqnarray*}
where, $E_{X}[h(X)]$ is the expectation of $h(X)$, where $h(X)$ is a
transformation of the random variable $X$.  So,
\begin{eqnarray*} 
E&\leq& \sum_{i=1}^N E_{X_i} [\log(f_i)] - \sum_{i=1}^N \sum_j w_{ij} E_{X_i}[\log(g_j)] \\
&=& \sum_{i=1}^N \sum_{j=1}^r w_{ij} E_{X_i} [\log(f_i)] - \sum_{i=1}^N \sum_{j=1}^r w_{ij} E_{X_i}[\log(g_j)] \\
&=& \sum_{i=1}^N \sum_j w_{ij} E_{X_i} [\log(f_i) - \log(g_j)] \\
&=& \sum_{i=1}^N \sum_j w_{ij} \text{KL}(f_i, g_j)
\end{eqnarray*}
So, $E \leq \sum_{i=1}^N \sum_j w_{ij} \text{KL}(f_i, g_j)$. Since,
both $E$ and $\sum_{i=1}^N \sum_j w_{ij} \text{KL}(f_i, g_j)$ attain
minima at the same value (equal to $0$), we can minimize $
\sum_{i=1}^N \sum_j w_{ij} \text{KL}(f_i, g_j)$ instead of $E$. Using
a Lagrange multiplier $r_i$ for each constraint $\sum_j w_{ij} = 1$, and
$\gamma_{ij}$ for each constraint $w_{ij} \geq 0$, we get the following
function
$$
\sum_j w_{ij} \text{KL}(f_i, g_j) + \sum_{i=1}^N r_i (1 - \sum_j w_{ij}) - \sum_{i, j} \gamma_{ij} w_{ij}
$$
We minimize the above objective function and add the KKT conditions
$$
\gamma_{ij} w_{ij} = 0
$$
to get, 
\begin{align}
\label{eq6}
\text{KL}(f_i, g_j) = \left\{\begin{array}{lr}
        r_i + \gamma_{ij}, & \text{if } w_{ij} = 0\\
        r_i, & \text{if } w_{ij} > 0
        \end{array}\right.
\end{align}
As each $\gamma_{ij} \geq 0$, this concludes the proof. 
\end{proof}
A straightforward Corollary of the above theorem is as follows:
\begin{corollary}
\label{cor1}
The objective function $E$ is bounded above by $\sum_{i=1}^N r_i$, i.e., $E \leq \sum_{i=1}^N r_i$.
\end{corollary}
\begin{proof}
From Theorem \ref{thm4}, we know that, $E \leq \sum_{i=1}^N \sum_j
w_{ij} \text{KL}(f_i, g_j)$. From Equation \ref{eq6}, we can now get
$\sum_j w_{ij} \text{KL}(f_i, g_j) = r_i$, $\forall i$. Thus the
Corollary holds.
\end{proof}
%\setlength{\columnsep}{1pt}
%\begin{wrapfigure}{r}{5cm} 
  %\centering
    %\includegraphics[scale=0.45]{sdl.png}
      %         \caption{Illustrative figure}\label{fig1}
%\end{wrapfigure}
\begin{figure}[!ht]
 \centering
    \includegraphics[scale=0.55]{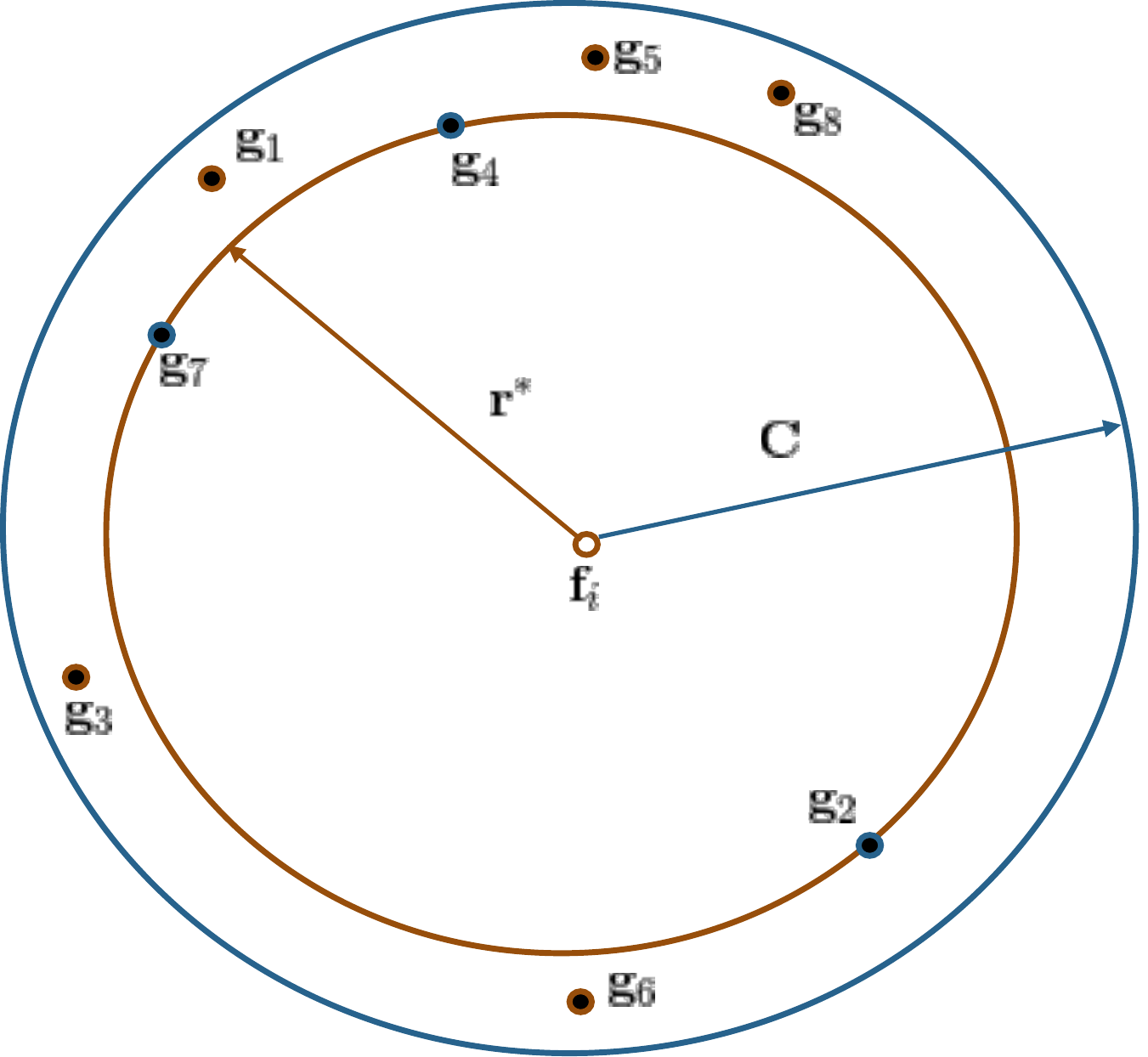}
\caption{Illustrative figure for Theorem \ref{thm5}. Blue and brown
  circles are atoms with non-zero and zero weights
  respectively.}\label{fig0.5}
\end{figure}
We can see that the dictionary elements, $g_j$, for which the
associated weights are positive, are exactly at the same distance
$r_i$ from the density $f_i$. Corollary \ref{cor1} implies that
solving the objective function in Equation \ref{eq5} yields a ``tight
cluster'' structure around each $f_i$, as minimizing $E$ is equivalent
to minimizing each $r_i$.
\begin{corollary}
\label{cor2}
Let, $f_i$ be well approximated by a single dictionary element
$g_{l}$. Further assume that $g_l$ is a convex combination of a set
of dictionary atoms, i.e., $g_l = \sum_{k=1}^{r_1} w_{ij_k}
g_{j_k}$. Without loss of generality (WLOG), assume that $w_{ij_k} > 0$,
$\forall k$. Let, $r_i = \text{KL}(f_i, g_l)$ and $\hat{r}_i =
\text{KL}(f_i, g_{j_k})$, $\forall k=1, \cdots, r_1$. Then, $r_i <
\hat{r}_i$.
\end{corollary}
\begin{proof}
Using the hypothesis in Theorem \ref{thm4}, we have, 
\begin{eqnarray*}
r_i &=& \int f_i(x)\log(f_i(x)) dx - \int f_i(x) \log(g_l(x)) dx  \\
& < & \left[ \int f_i(x) \log(f_i(x)) dx - \int f_i(x) \sum_{k=1}^{r_1} w_{ij_k} \log(g_{j_k}) dx\right] \\
&=& \sum_{k=1}^{r_1} w_{ij_k} \text{KL}(f_i, g_{j_k}) \\
&=& \hat{r}_i
\end{eqnarray*}
Hence, $r_i < \hat{r}_i$. Using Corollary \ref{cor1}, we can see that,
in order to represent $f_i$, the objective function is to minimize
$r_i$. Thus, using Corollary \ref{cor1}, we can say that a sparse set
of weights, i.e., corresponding to $g_l$, is preferable over a set of
non-zero weights, i.e., corresponding to a set of $\{g_{j_k}\}$. 
\end{proof}

Now, we will state and prove the second key result namely, a theorem
which states that our proposed algorithm yields non-zero number of
atoms whose corresponding weights are zero i.e., $k-sparsity$ for some
$k > 0$.
\begin{theorem}
\label{thm5}
Let $\mathfrak{S}_i = \left\{w_{ij} | w_{ij}=0\right\}$, then, with
probability $1$, the cardinality of $\mathfrak{S}_i$, i.e.,
$|\mathfrak{S}_i| >0$, for all $i$.
\end{theorem}
\begin{proof}
Let $\mu^*$ be the probability measure on $\mathfrak{D}$. let
$\mathcal{B}(f, r)$ denote a closed ball of radius $r$ centered at $f
\in \mathfrak{D}$, we assume that the measure is bounded, i.e.,
$\exists$ constants, $\kappa_1>0$ and $\kappa_2>0$ such that, $
\left(\kappa_1 r\right)^n \leq \mu^*\left(\mathcal{B}(f, r)\right)
\leq \left(\kappa_2 r\right)^n $, for all $0<r\leq 1$. Let us assume,
$\mathfrak{G}$ is an $\epsilon$-separated set for some $0<\epsilon
\leq 1$. Furthermore, assume that $\mathfrak{F}$ has finite variance,
i.e., $\exists C$ such that $\forall i, j$, $\text{KL}\left(f_i,
f_j\right) \leq C$, we will call this radius $C$ (closed) ball as the
\emph{data ball}. Let the optimum value of $E/N$ be $r^*$, i.e., $ r^*
= \max_i \text{KL}\left(f_i, \hat{f}_i\right)$ (see Figure
\ref{fig0.5}). Now, for a given $i$, consider $\mathcal{B}\left(f_i,
r^*\right)$, from Theorem \ref{thm4}, we know that if for some $j$,
$w_{ij} >0$, then, $\text{KL}(f_i, g_j) = r^*$, else, $\text{KL}(f_i,
g_j) > r^*$.

Thus, $\mathfrak{S}_i$ can be rewritten as, $\mathfrak{S}_i = \left\{
g_j | \text{KL}(f_i, g_j) > r^*\right\}$. Let, $N\left(r^*, C\right)$
be the number of $g_j$s in $\mathcal{B}(f_i, C) \setminus
\mathcal{B}(f_i, r^*)$. Then, $N\left(r^*, C\right)$ follows a Poisson
distribution with rate $\lambda = \mu^*\left(\mathcal{B}(f,
\epsilon/2)\right)$. Hence, it is easy to see that $|\mathfrak{S}_i| =
E\left[N\left(r^*, C\right)\right]$. Now,
\begin{align*}
E\left[N\left(r^*, C\right)\right] &=  \mu^*\left(\mathcal{B}(f, \epsilon/2)\right) \left(\frac{2(C-r^*)}{\epsilon}\right)^n \\
&\geq \left(\frac{\kappa_1 \epsilon}{2}\right)^n \left(\frac{2(C-r^*)}{\epsilon}\right)^n \\
&= \left(\kappa_1 (C-r^*)\right)^n
\end{align*}
Since we are reconstructing $f_i$ as a convex combination of $g_j$s,
the only case when $r^* = C$ occurs when all $f_i$s lie on the
boundary of the \emph{data ball}. Let $\mathcal{T} = \left\{f_i| f_i
\in \mathfrak{F}, f_i \text{is on the boundary of \emph{data
    ball}}\right\}$, clearly, $\mu^*\left(\mathcal{T}\right) =
0$. Hence, with probability $1$, $C-r^*>0$. Now, as $\kappa_1>0$, we
can say with probability $1$, $|\mathfrak{S}_i| = E\left[N\left(r^*,
  C\right)\right] >0$. Since $i$ is arbitrary, the claim holds. This
comletes the proof.
\end{proof}

\emph{Theorem \ref{thm5} states that our proposed algorithm yields
  $k$-sparse atoms, for some $k>0$}.
\vspace*{12pts}

{\bf Comment on using Hellinger distance:} On the space of densities,
one can define Hellinger distance (denoted by $d_{L2}$) as follows:
Given $f, g \in \mathfrak{D}$, one can use square root parametrization
to map these densities on the unit Hilbert sphere, let the points be
denoted by $\bar{f}$, $\bar{g}$. Then, one can define the $\ell_2$
distance between $\bar{f}$ and $\bar{g}$ (the Hellinger distance
between $f$ and $g$) as $d^2_{L2}\left(f,g\right) = \frac{1}{2}\:\int
\left(\bar{f}-\bar{g}\right)^2$. One can easily see that the above
expression is equal to $d^2_{L2}\left(f,g\right) = \left( 1 - \int
\left(\bar{f}\bar{g}\right) \right)$. This metric is the chordal
metric on the hypersphere, and hence an {\it extrinsic} metric. We
now replace $\text{KL}$ divergence by the Hellinger distance in our
objective function in Eq. \ref{eq5}. The modified objective function is
given in Eq. \ref{eq115}.

\begin{align}
\label{eq115}
\displaystyle \argmin_{\mathfrak{G}^*, W^*} \:\:\:\:E &  = \sum_{i=1}^N d^2_{L2}\left(f_i, \hat{f}_i\right) \\
\text{subject to} & \: \: \: w_{ij} \geq 0, \forall i,j \\
& \sum_j w_{ij} = 1, \forall i.
\end{align}

One can easily show that the above analysis of sparsity also holds
when we replace the $\text{KL}$ divergence by the Hellinger distance
(as done in Eq. \ref{eq115}). The following Theorem (without proof)
states this result.

\begin{theorem}
\label{thm114}
Let $\mathfrak{G} = \{g_j\}$ and $W = [w_{ij}]$ be the solution of the
objective function $E$ in Equation \ref{eq115}. Then,
\begin{align*}
\left(\forall j\right), d^2_{L2}(f_i, g_j) \geq r_i , \text{ where } r_i = \sum_k w_{ik} d^2_{L2}(f_i, g_j)
\end{align*}
Then, analogous to Corollary \ref{cor2}, it can be easily shown that the
set of weights is sparse.
\end{theorem}
\subsubsection{DL and SC on the {\it manifold of SPD matrices}}
\label{sec112}

Let, the manifold of $n\times n$ SPD matrices be denoted by
$\mathcal{P}_n$. We will use the following notations for the rest of
the paper. On $\mathcal{P}_n$,
\begin{itemize}
 \item $\mathcal{C}$ be a dictionary with $r$ atoms $C_1$, $\cdots,$
   $C_r$, where each $C_i \in \mathcal{P}_n$.
 \item $\mathcal{X} = \{X_i\}_{i=1}^N$ $\subset \mathcal{P}_n$ be a
   set of data points.
\item $w_{ij}$ be nonnegative weights corresponding to the $i^{th}$
  data point and the $j^{th}$ atom, $i \in \{1,\cdots, N\}$ and $j \in
  \{1, \cdots, r\}$.
\end{itemize}
We now extend the DL and SC formulation to $\mathcal{P}_n$. Note that,
a point, $C \in \mathcal{P}_n$ can be identified with a Gaussian
density with zero mean and covariance matrix $C$. Hence, it is natural
to extend our information theoretic DL \& SC framework from a
statistical manifold to $\mathcal{P}_n$. Recall that the symmetrized
$\text{KL}$ divergence between two densities $f$ and $g$ can be
defined by the JSD in Equation \ref{eq11}. Using the square root of
the JSD, one can define a ``distance'' between two SPD matrices on
$\mathcal{P}_n$ (the quotes on distance are used because JSD does not
satisfy the traingle inequality for it to be a distance
measure). Similar to Equation \ref{eq4}, we can analogously define the
{\it symmetrized weighted KL center}, denoted by $M_{\text{KL}}$, as
the minimizer of the sum of symmetrized squared KL divergences. Given,
$\mathcal{X} = \{X_i\}_{i=1}^N$, we can define the {\it symmetrized
  KL-center} of $\mathcal{X}$ as follows \cite{wang2005dti}
$$
M_{\text{KL}}(\mathcal{X}) = \sqrt{B^{-1}}\sqrt{\sqrt{B}A\sqrt{B}}\sqrt{B^{-1}}
$$ where $A = \frac{1}{N} \sum_i X_i$, $B = \frac{1}{N} \sum_i
X_i^{-1}$. We can extend the above result to define the {\it
  symmetrized weighted KL-center} via the following Lemma.
\begin{lemma}
On $\mathcal{X} = \{X_i\}_{i=1}^N$ with weights $\{w_i\}_{i=1}^N$, the
{\it symmetrized weighted KL-center}, $M_{\text{KL}}(\mathcal{X},
\{w_i\})$ is defined as
$$
M_{\text{KL}}(\mathcal{X}, \{w_i\}) =  \sqrt{B^{-1}}\sqrt{\sqrt{B}A\sqrt{B}}\sqrt{B^{-1}}
$$
where $A = \frac{1}{\sum_j w_j} \sum_i w_i X_i$, $B = \frac{1}{\sum_j w_j} \sum_i w_i X_i^{-1}$
\end{lemma}
Analogous to Equation \ref{eq5}, we can define our formulation
for DL and SC on $\mathcal{P}_n$ as follows:
\begin{align}
\label{eq15}
\displaystyle \argmin_{\mathcal{C}^*, W^*} \:\:\:\:E &  = \sum_{i=1}^N \text{J}(X_i, \hat{X}_i) \\
\text{where} & \: \: \: \hat{X}_i = M_{\text{KL}}(\mathcal{C}, \{w_{ij}\}_{j=1}^r) \\
\text{subject to} & \: \: \: w_{ij} \geq 0, \forall i,j \\
& \sum_j w_{ij} = 1, \forall i.
\end{align}
Here $\text{J}(X, \hat{X})$ is the {\it symmetrized-KL} also known as the
J-divergence and is defined as:
$$
\text{J}(X,\hat{X}) = \frac{1}{4} \left[X^{-1}\hat{X} + \hat{X}^{-1}X - 2n\right]
$$

%On $\mathcal{P}_n$, we have used the $GL(n)$-invariant metric, where $GL(n)$ denotes the group of $n \times n$ non-singular matrices.$\mathcal{P}_n$ admits a group action defined by, $\forall h \in GL(n)$, $\forall C \in \mathcal{P}_n$, $C[h] = hCh^t$. Let, $U, V \in T_C \mathcal{P}_n$, then the $GL(n)$-invariant metric can be defined as $ <U, V>_{C} = trace(C^{-1/2}UC^{-1}VC^{-1/2}) $ The distance induced by the metric is defined as, $d(C, \hat{C}) = trace(Log(C^{-1}\hat{C}))$, here $Log$ is the matrix logarithm. One can use any distance or divergence function for $d$ in Equation \ref{eq15}, but we have used the $GL(n)$-invariant metric.  

\begin{algorithm}
%\footnotesize{
    \KwIn{$\mathcal{X} = \{X_i\}_{i=1}^N$ $\subset \mathcal{P}_n$, $\eta > 0$, $\epsilon>0$}
    \KwOut{$\mathcal{C} = \{C_j\}_{j=1}^r$ $\subset \mathcal{P}_n$, $W = [w_{ij}] \geq 0$}
    
    Initialize $\mathcal{C}$ by using the k-means algorithm on$ \mathcal{P}_n$\;
    Initialize $W$ randomly using random non-negative numbers from $[0,1]$\;
    \For{$i = 1, \cdots, N$}{
     Normalize the vector $w(i,.)$ so that it sums to $1$ \;
    }
    
    $flag \leftarrow 1$\; Compute the objective function, $E$ using
    Equation \ref{eq15}\; $E^\text{old} \leftarrow E$\; $\text{iter}
    \leftarrow 1$\; $\lambda(1) \leftarrow 1$\; $Y^W \leftarrow W$\;
    $Y^C_j \leftarrow C_j$, $\forall j$\; 
\While{flag $=1$}{ Perform an alternating step optimization by alternating
  between $\mathcal{C}$ and $W$ using the accelerated gradient descent
  method \cite{bubeck2014theory}\; $\lambda(\text{iter}+1) \leftarrow
  \frac{1+\sqrt{1+4\lambda(\text{iter})^2}}{2}$\; $\gamma(\text{iter})
  \leftarrow \frac{1-\lambda(\text{iter})}{\lambda(\text{iter}+1)}$\;
  $nY^W \leftarrow W - \eta*\frac{\nabla E}{\nabla W}$\; $W \leftarrow
  (1-\gamma(\text{iter}))nY^W + \gamma(\text{iter})Y^W$\; Using $W$
  update $C_j$ using the following steps, $\forall j$\; $nY^C_j
  \leftarrow Exp_{C_j} \left(-\eta \nabla E(C_j)\right)$\; $C_j
  \leftarrow
  Exp_{nY^C_j}\left(\gamma(\text{iter})Log_{nY^C_j}Y^C_j\right)$\;
       
       Recompute the objective function, $E$ using new $\mathcal{C}$
       and $W$, using Eq. \ref{eq15}\;
               
        \If{$|E - E^\text{old}| < \epsilon$}{
            {\it flag} $\leftarrow 0$\;
          }  
       
       $\text{iter} \leftarrow \text{iter} + 1$
        }
       
    \caption{{The SDL algorithm}}
    \label{alg1}
   % }
\end{algorithm}

Now, we present an algorithm for DL and SC on $\mathcal{P}_n$ that
will henceforth be labeled as the {\it information theoretic
  dictionary learning and sparse coding} (SDL) algorithm. We use an
alternating step optimization procedure, i.e., first learn $W$
with $\mathcal{C}$ held fixed, and then learn $\mathcal{C}$ with $W$
held fixed. We use the well known Nesterov's accelerated gradient
descent \cite{bubeck2014theory} adapted to Riemannian manifolds for
the optimization. The algorithm is summarized in the Algorithm block
\ref{alg1}.

In the algorithm, after the initialization steps up to line $13$, we do
an alternative step optimization between $\mathcal{C}$ and $W$. Line
$15$-$18$ are updates of $W$ using the accelerated gradient
descent. In line $20$, we use the Riemannian gradient descent to
map the gradient vector on the manifold (to get $nY^C$) using
Riemannian Exponential map ($Exp$) \cite{do1992riemannian}. Then, we
update $C_j$ using the Riemannian accelerated gradient descent steps
by first lifting $Y^C$ on to the tangent space anchored at $nY^C$ (using
Riemannian inverse exponential map ($Log$) and then map it on the
manifold using $Exp$ map. Then, we recompute the error using the
updated $\mathcal{C}$ and $W$ and then iterate.

\section{Experimental Results}\label{sec4}

In this section, we present experimental results on several real data
sets demonstrating the performance of our algorithm, the {\it SDL}. We
present two sets of experiments showing the performance in terms of
(1) reconstruction error and achieved sparsity on a statistical
manifold and, (2) classification accuracy and achieved sparsity on the
manifold of $n \times n$ SPD matrices, $P_n$. Though the objective of
a DL and SC algorithm is to minimize reconstruction error, due to the
common trend (in literature) of using classification accuracy as
measure, we report the classification accuracy measure on popular
datasets for data on $P_n$. But since the main thrust of the paper is
a novel DL and SC algorithm on a statistical manifold, we present
reconstruction error experiments in support of the algorithm
performances. All the experimental results reported here were obtained
on a desktop with a single 3.33 GHz Intel-i7 CPU with 24 GB RAM. We
did not compare our work with the algorithm proposed in \cite{Xie2013}
since for a moderately large data, their publicly available code makes
comparisons computationally infeasible.

\subsection{Experimental results on the statistical manifold}

In order to demonstrate the performance of {\it SDL} on {\it MNIST
  data} \cite{lecun1998gradient}, we randomly chose $100$ images from
each of the $10$ classes. We then represent each image as a
probability vector as follows. We consider the image graph $Z =
\left(x, y, I(x,y)\right)_{(x,y)}$ and take $Z$ as the random vector
in $\mathbf{R}^2 \times \mathbf{Z}_{+}$. The probability mass function
(p.m.f.) of $Z$ is given as: $\text{Pr}\left(Z =
\left(x_0,y_0,I(x_0,y_0)\right)\right) =
\frac{I(x_0,y_0)}{\sum_{x,y}I(x,y)}$. Now, each image is mapped as a
probability vector (or discrete density) and we use our formulation of
DL and SC to reconstruct the images. Note that, the reconstruction is
upto a scale factor.

In order to compare, we used two popular methods, namely (i) the K-SVD
based method in \cite{aharon2005k} (we chose number of atoms to be
twice the number of classes and chose $50\%$ sparsity) and (ii) the
Log-Euclidean sparse coding (LE-SC) method \cite{Guo2013}. Both of
these methods assume that the data lie in a vector space. As the
objective functions for these methods are different, hence we use {\it
  mean squared error} (MSE) as a metric to measure reconstruction
error. We also report the achieved sparsity by these methods.

\begin{table}
\begin{center}
\begin{tabular}{|c|c|c|c|c|c|}
\hline
\multicolumn{2}{|c|}{SDL} & \multicolumn{2}{c|}{K-SVD  \cite{aharon2005k}} & \multicolumn{2}{c|}{LE-SC \cite{Guo2013}} \\
\cline{1-6}
MSE  & $\varsigma (\%)$ & MSE & $\varsigma (\%)$ & MSE & $\varsigma (\%)$ \\
\hline
\textbf{0.052} & 97.8 & 0.070 & 98.3 & 0.098 & \textbf{98.5} \\
\hline
\end{tabular}
\end{center}
\caption{Comparison results on MNIST data}
\label{tab0}
\end{table}

From the Table \ref{tab0}, it is evident that, though K-SVD and LE-SC
perform better in terms of sparsity, SDL achieved the best
reconstruction error while retaining sparse atoms. Some reconstruction
results are also shown in Fig. \ref{fig0}.  The results clearly
indicate that SDL gives ``sharper'' reconstruction compared to the two 
competing methods. This is because the formulation of SDL respects
the geometry of the underlying data while other two methods do not.

\begin{figure}[!ht]
\centering
  \begin{minipage}[b]{0.45\textwidth}
  \centering
    \includegraphics[width=0.8\textwidth]{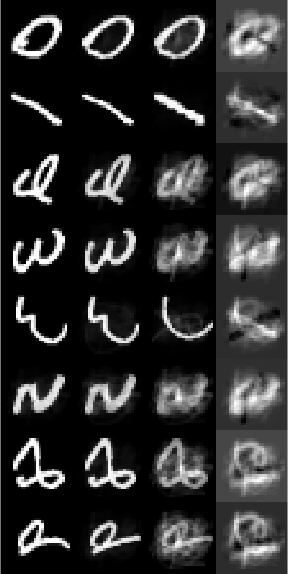}
  \end{minipage}
\caption{Reconstruction of MNIST data (left to right) (a)
  original data (b) SDL (c) K-SVD (d) LE-SC}
\label{fig0}
\end{figure}

\subsection{Experimental results on $P_n$}

Now, we will demonstrate the effectiveness of our proposed method {\it
  SDL} compared to the state-of-the-art algorithms on classification
using the SCs as features for the classification problem on the
manifold $P_n$ of SPD matrices. We report the classification accuracy
to measure the performance in the context of classification
experiments. Moreover, we also report a measure of sparsity, denoted
by $\varsigma$, which captures the percentage of the elements of $W$
that are $\leq 0.01$.  We performed comparisons to three
state-of-the-art methods namely, (i) Riemannian sparse coding for SPD
matrices (Riem-SC) \cite{Cherian2014}, (ii) Sparse coding using the
kernel defined by the symmetric Stein divergence (K-Stein-SC)
\cite{Harandi2012}, (iii) Log-Euclidean sparse coding (LE-SC)
\cite{Guo2013}. For the LE-SC, we used the highly cited SPAMS toolbox
\cite{mairal2009online} to perform the DL and SC on the tangent space.

We tested our algorithm on three commonly used (in this context) and
publicly available data sets namely, (i) the Brodatz texture data
\cite{brodatz1966textures}, (ii) the Yale ExtendedB face data
\cite{KCLee05}, and (iii) the ETH80 object recognition data
\cite{eth80}. The data sets are described in detail below. From each
of data set, we first extract $\mathcal{P}_n$ valued features. Then,
{\it SDL} learns the dictionary atoms and the sparse codes. Whereas,
for the Riem-SC and the kStein-SC, we used k-means on $\mathcal{P}_n$
and used the cluster centers as the dictionary atoms. For the
Log-Euclidean sparse coding, we used the Riemannian Inverse
Exponential map \cite{do1992riemannian} at the Fr{\'e}chet mean (FM)
of the data and performed a Euclidean DL and SC on the tangent space
at the FM. For classification, we used the $\nu-SVM$
\cite{scholkopf2002learning} on the sparse codes taken as the
features. The SVM parameters are learned using a cross-validation
scheme.  \ \\

\vspace{0.2cm} \ \\ {\bf Brodatz texture data:} This dataset contains
$111$ texture images. We used the same experimental setup as was used
in \cite{sivalingam2010tensor}. Each image is of $256 \times 256$
dimension and we first partitioned each image into $64$
non-overlapping blocks of size $32\times 32$. From each block, we
computed a $5\times 5$ covariance matrix $FF^t$, summing over the
block, where $F = (I, |\frac{\partial I}{\partial x}|, |\frac{\partial
  I}{\partial y}|, |\frac{\partial^2 I}{\partial x^2}|,
|\frac{\partial^2 I}{\partial y^2}|)^t$. The matrix $FF^t$ is
symmetric positive semidefinite. To make this matrix an SPD matrix, we
add $\sigma I$ to it, where $\sigma$ is a small positive real
number. Thus, the covariance descriptor from each image lies on
$\mathcal{P}_5$. For this data, we consider each image as a class,
resulting in a $111$ class classification problem. As DLM is
computationally very expensive, this $111$ class classification is
infeasible using this method, hence we also randomly selected $16$
texture images and performed classification on $16$ classes to
facilitate this comparison. We took the number of dictionary atoms
($r$) to be $555$ and $80$ for the $111$ classes and $16$ classes
respectively.
\begin{figure}[!ht]
\centering
  \begin{minipage}[b]{0.49\textwidth}
  \centering
    \includegraphics[width=\textwidth]{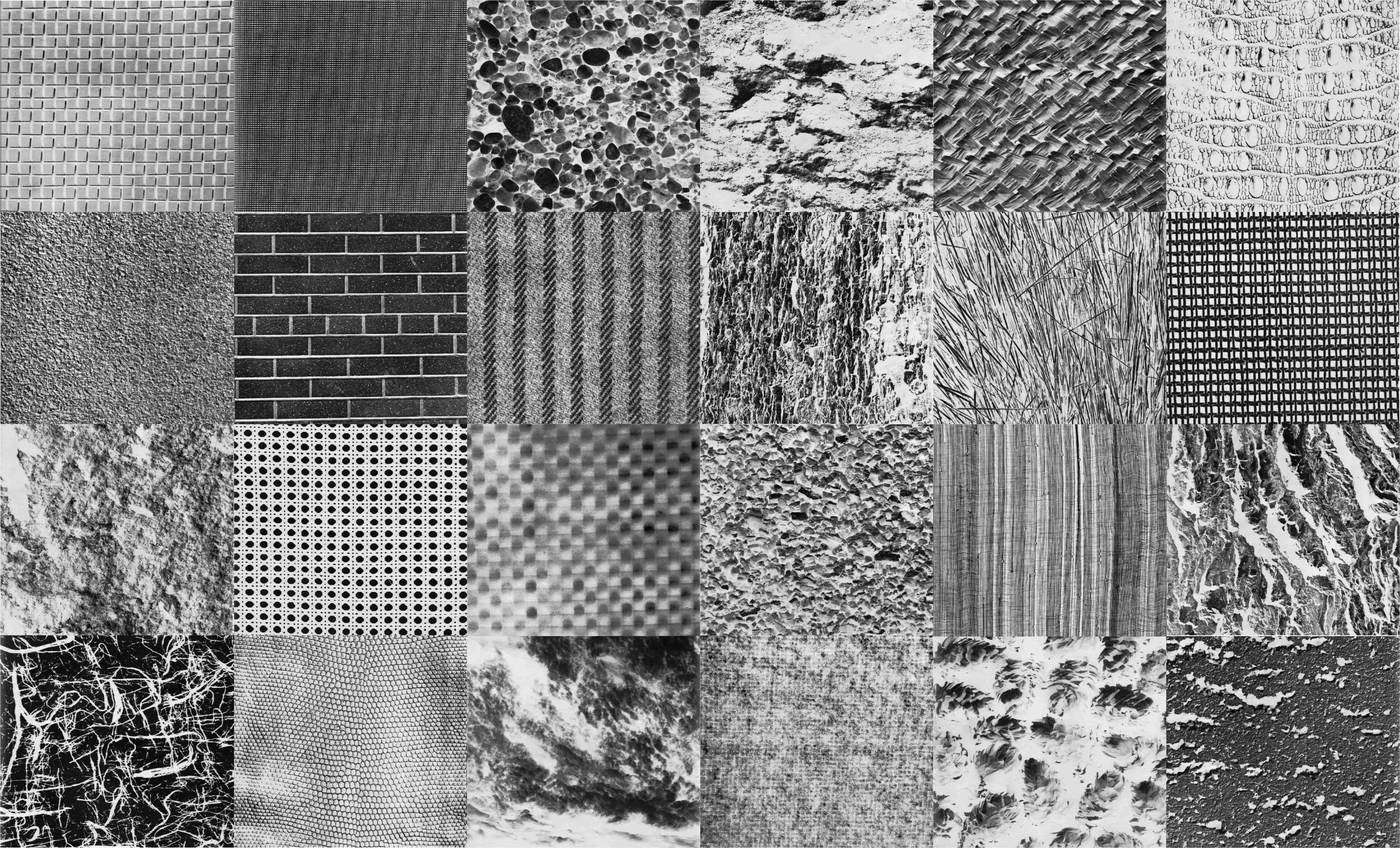}
  \end{minipage}
  \caption{Brodatz data samples. }
    \label{fig1}
    \end{figure}

\ \\
\vspace{0.2cm} \ \\ {\bf Yale face data:} This YaleExtendedB face data
set contains $16128$ face images acquired from $28$ human subjects
under varying pose and illumination conditions. We randomly fixed a
pose and for that pose consider all the illuminations, leading to
$252$ face images taken from $28$ human subjects. We used a similar
type of experimental setup as described in
\cite{chakraborty2015iccv}. From each face image, we construct a SIFT
descriptor \cite{sift} and take the first $4$ principal vectors of
this descriptor. Thus, each image is identified with a point on the
Grassmann manifold of appropriate dimension. And then, inspired by the
isometric mapping between the Grassmannian and $\mathcal{P}_n$
\cite{huang2015projection}, we construct the covariance descriptor
from the aforementioned principal vectors. Here, we used $84$
dictionary atoms.
\begin{figure}[!ht]
\centering
  \begin{minipage}[b]{0.49\textwidth}
  \centering
    \includegraphics[width=\textwidth]{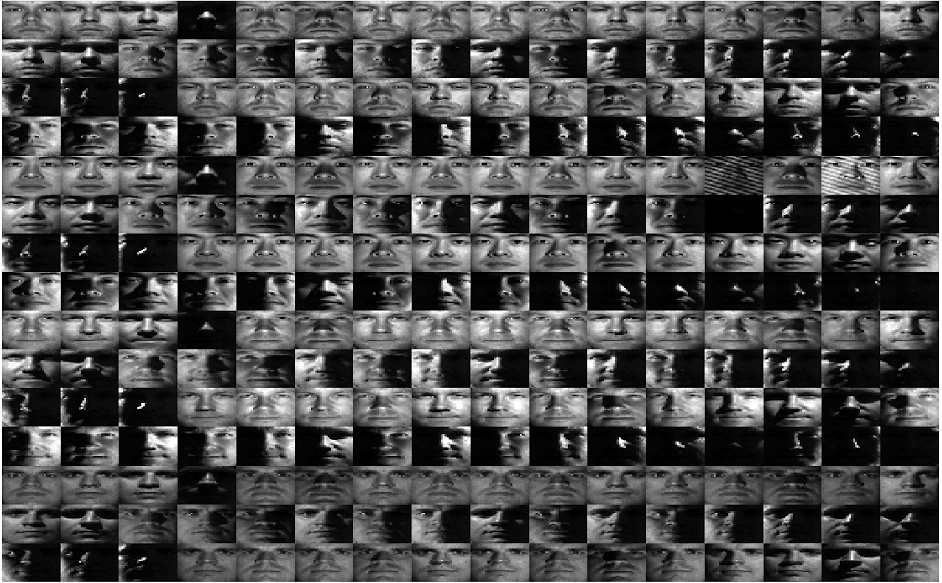}
  \end{minipage}
  \caption{Yale face data samples. }
 \label{fig2}
 % \vspace{-2em}
\end{figure}
  \ \\

{\bf ETH80 object recognition data:} This dataset contains $8$
different objects, each having $10$ different instances from $41$
different views resulting in $3280$ images. We first segment the
objects from each image using the provided ground truth. We used both
texture and edge features to construct the covariance matrix. For the
texture feature, we used three texture filters
\cite{laws1980rapid}. The filter bank is $[H_1H_1^t, H_2H_2^t,
  H_3H_3^t]$, where $H_1 = [1,2,1]^t$, $H_2 = [-1,0,1]^t$, $H_3 =
     [-1,2,-1]^t$. In addition to the three texture features, we used
     the image intensity gradient and the magnitude of the smoothed
     image using Laplacian of the Gaussian filter. We used $40$
     dictionary atoms for this data.
\begin{figure}[!ht]
\centering
  \begin{minipage}[b]{0.49\textwidth}
    \includegraphics[width=\textwidth]{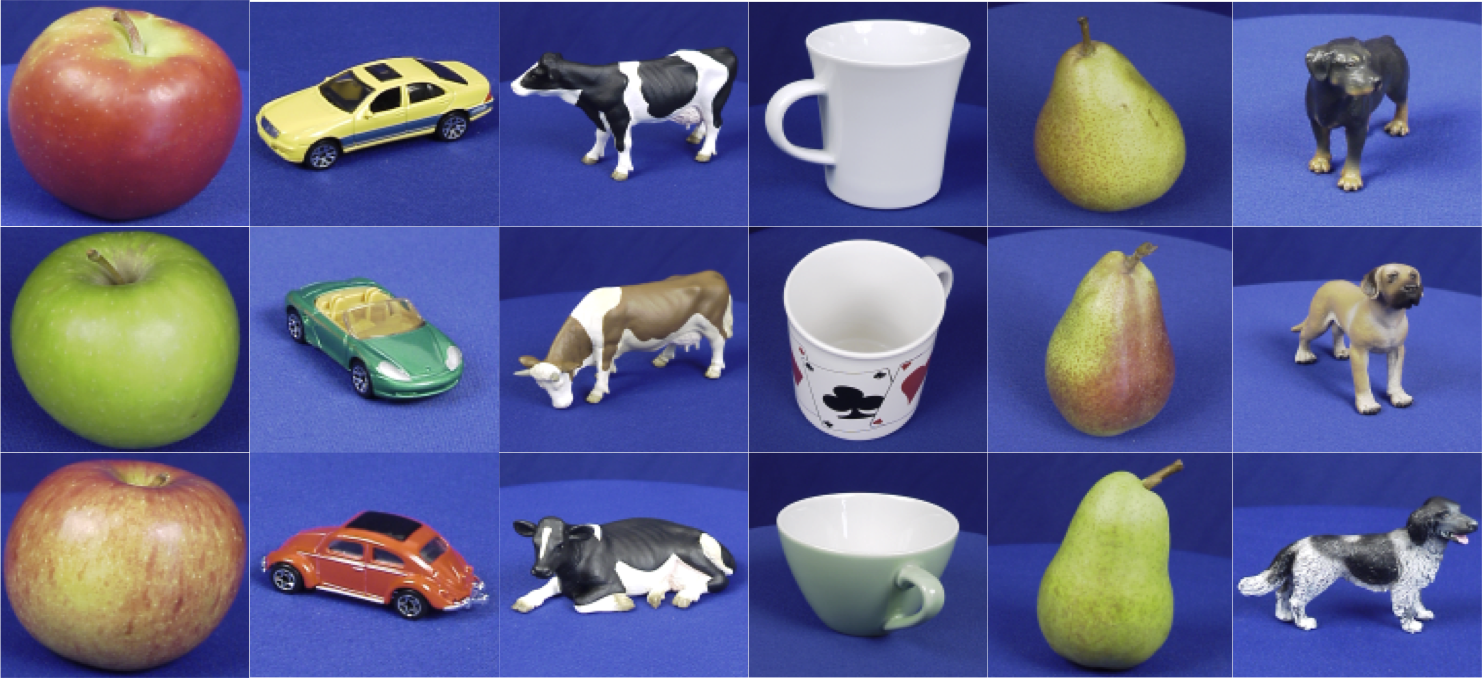}
  \end{minipage}
  \caption{Eth80 data samples. }
\label{fig3}
\end{figure}
  \ \\ Performance comparisons are depicted in Tables
  \ref{tab1}-\ref{tab2} respectively. All of these three methods are
  intrinsic, i.e., the DL and SC are tailored to the underlying
  manifold, i.e., $\mathcal{P}_n$. In order to compute the
  reconstruction error, we have used the intrinsic affine invariant
  metric on $\mathcal{P}_n$. From the tables, we can see that SDL
  yields the best sparsity amongst the three manifold-valued methods
  (excluding LE-SC). Furthermore, on the Yale-face data set, the SDL
  is computationally most efficient algorithm when compared to Riem-SC
  and kStein-SC respectively. In terms of reconstruction error, our
  proposed method outperforms it's competitors. Note that, for
  kStein-SC, computing the reconstruction error is not meaningful as
  they solved the DLSC problem on the Hilbert space after using a
  kernel mapping.

\begin{table*}
\begin{center}
\begin{tabular}{|c|c|c|c|c|c|c|c|c|}
\hline
\multirow{2}{1em}{Data} & \multicolumn{2}{c|}{SDL} & \multicolumn{2}{c|}{Riem-SC \cite{Cherian2014}} & \multicolumn{2}{c|}{kStein-SC \cite{Harandi2012}} & \multicolumn{2}{c|}{LE-SC \cite{Guo2013}} \\
\cline{2-9}
& acc. (\%) & $\varsigma (\%)$ & acc. (\%) & $\varsigma (\%)$ & acc. (\%) & $\varsigma (\%)$ & acc. (\%) & $\varsigma (\%)$ \\
\hline
Brodatz & \textbf{95.02} & 95.96  & 66.21 & 61.54 & 88.57 & 95.40 & 60.25 & \textbf{98.52} \\
Yale & \textbf{68.68} & 92.95 & 59.98 & 69.10 & 53.55 & 97.58 & 8.35 & \textbf{99.78} \\
Eth80 & \textbf{96.10} & 88.07 & 91.46 & 78.64 & 45.79 & 97.97 & 45.79 & \textbf{97.97} \\
\hline
\end{tabular}
\end{center}
\caption{Comparison results on three data sets in terms of classification accuracy and amount of sparsity}
\label{tab1}
\end{table*}

\begin{table*}
\begin{center}
\begin{tabular}{|c|c|c|c|c|c|c|c|c|}
\hline
\multirow{2}{1em}{Data} & \multicolumn{2}{c|}{SDL} & \multicolumn{2}{c|}{Riem-SC \cite{Cherian2014}} & \multicolumn{2}{c|}{kStein-SC \cite{Harandi2012}} & \multicolumn{2}{c|}{LE-SC \cite{Guo2013}}  \\
\cline{2-9}
& recon. err. & Time(s) & recon. err. & Time(s) & recon. err. & Time(s) & recon. err. & Time(s) \\
\hline
Brodatz & \textbf{0.55} & 599.33  & 1.24 & 539.63 & N/A & 527.57 & 3.21 & \textbf{5.46} \\
Yale & \textbf{0.001} & 47.25 & 0.005 & 293.91 & N/A & 131.97 & 0.25 & \textbf{9.16} \\
Eth80 & \textbf{0.005} & 153.60 & 0.017 & 240.96 & N/A & 213.93 & 0.16 & \textbf{3.62} \\
\hline

\end{tabular}
\end{center}
\caption{Comparison results on three data sets in terms of reonstruction error and computation time}
\label{tab2}
\end{table*}

We also depict the comparative performance as a function of number of
dictionary atoms, for the four algorithms in Fig. \ref{fig4} (for
Brodatz data) and in Fig. \ref{fig5} (for the Yale face data
set). Here, we have shown the comparative performance in terms of
classification accuracy, reconstruction error and required CPU
time. For both these data sets, we can see the superior performance of
SDL over it's competitors in terms of classification accuracy and
sparsity. As the objective of any DL algorithm is to reconstruct the
samples, we have also shown the reconstruction error thereby depicting
the competitive performance of SDL over the other algorithms.
\begin{figure*}
\centering
  \begin{minipage}[b]{0.32\textwidth}
    \includegraphics[width=\textwidth]{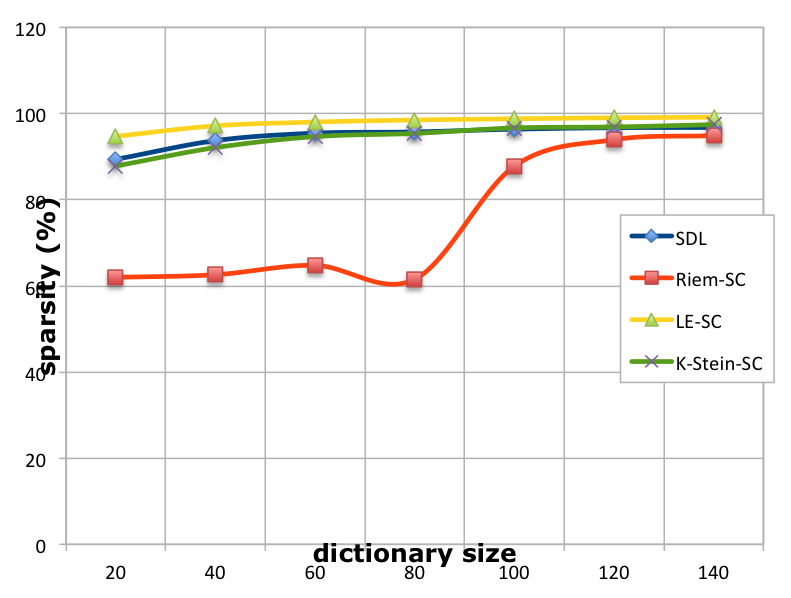}
  \end{minipage}
  \begin{minipage}[b]{0.32\textwidth}
    \includegraphics[width=\textwidth]{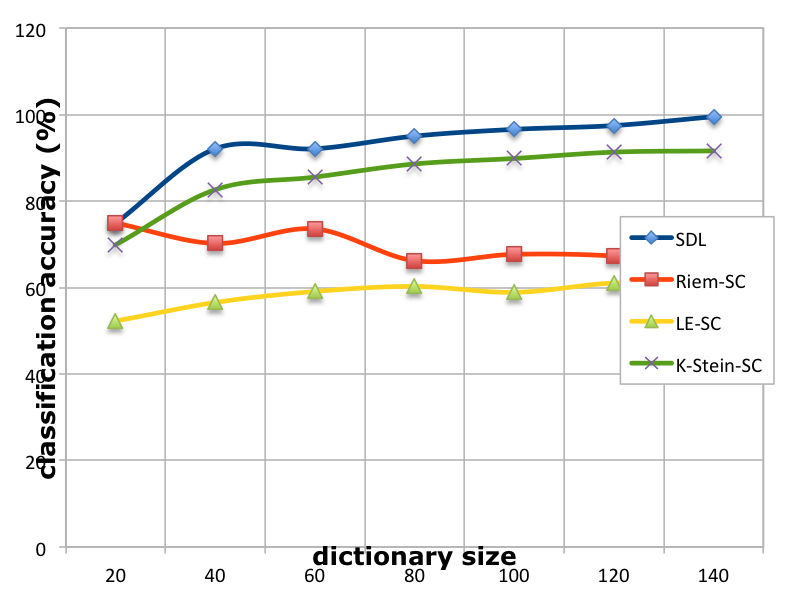}
  \end{minipage}
   \begin{minipage}[b]{0.32\textwidth}
    \includegraphics[width=\textwidth]{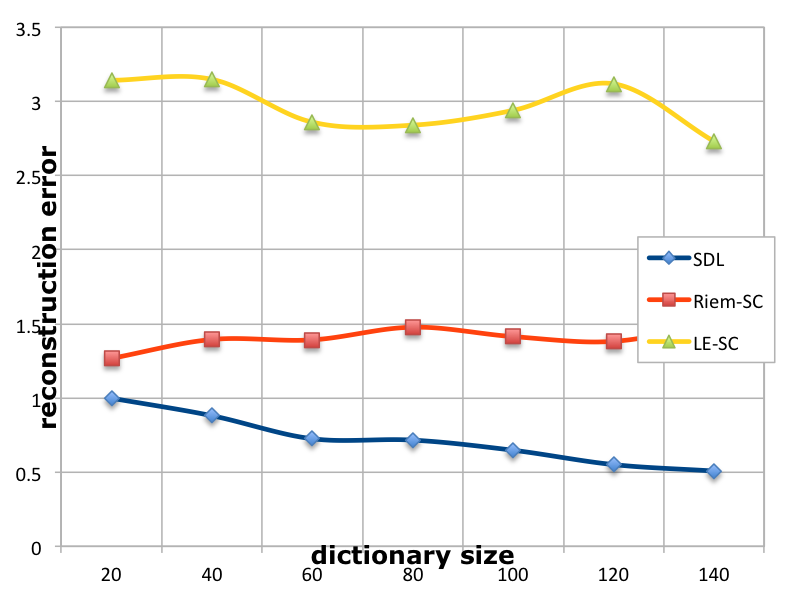}
  \end{minipage}
  \caption{Brodatz data results: \emph{Left: } Sparsity, \emph{Middle: } Classification accuracy and \emph{Right: } Reconstruction error with varying number of dictionary atoms.}
  \label{fig4} 
\end{figure*}

\begin{figure*}
\centering
  \begin{minipage}[b]{0.32\textwidth}
    \includegraphics[width=\textwidth]{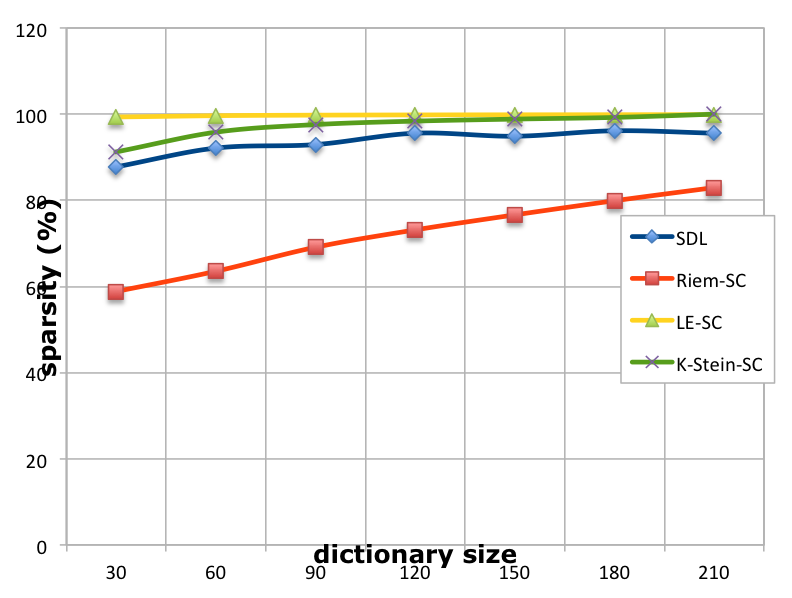}
  \end{minipage}
  \begin{minipage}[b]{0.32\textwidth}
    \includegraphics[width=\textwidth]{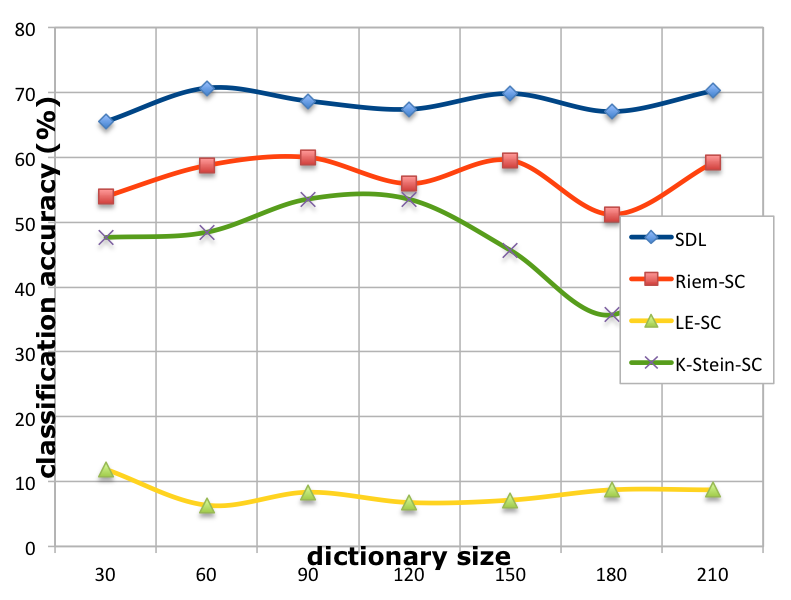}
  \end{minipage}
   \begin{minipage}[b]{0.32\textwidth}
    \includegraphics[width=\textwidth]{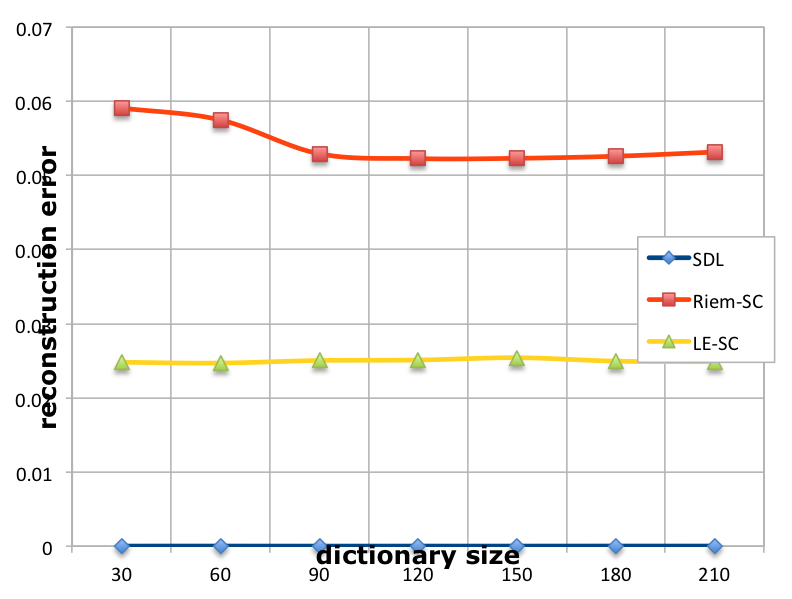}
  \end{minipage}
  \caption{Yale face data results: \emph{Left: } Sparsity, \emph{Middle: } Classification accuracy and \emph{Right: } Reconstruction error with varying number of dictionary atoms.}
  \label{fig5}
\end{figure*}

\section{Conclusions}\label{sec5}
In this paper, we presented an information theoretic dictionary
learning and sparse coding algorithm for data residing on a
statistical manifold. In the traditional dictionary learning approach
on a vector space, the goal is to express each data point as a sparse
linear combination of the dictionary atoms.  This is typically
achieved via the use of a data fidelity term and a term to induce
sparsity on the coefficients of the linear combination. In this paper,
we proposed an alternative formulation of the DL and SC problem for
data residing on statistical manifolds, where we do not have an
explicit sparsity constraint in our objective function. Our algorithm,
SDL, expresses each data point, which is a probability distribution,
as a weighted KL-center of the dictionary atoms. We presented a proof
that our proposed formulation yields sparsity without explicit
enforcement of this constraint and this result holds true when the
KL-divergence is replaced by the Hellinger distance between
probability densities. Further, we presented an extension of this
formulation to data residing on $\mathcal{P}_n$. A Riemannian
accelerated gradient descent algorithm was employed to learn the
dictionary atoms and an accelerated gradient descent algorithm was
employed to learn the sparse weights in a two stage alternating
optimization framework. The experimental results demonstrate the
effectiveness of the SDL algorithm in terms of reconstruction and
classification accuracy as well as sparsity.

\begin{center}
{\bf Acknowledgements}
\end{center}

This research was
    funded in part by the NSF grants IIS-1525431 and IIS-1724174 to
    BCV. We thank Dr. Shun-ichi Amari for his insightful comments on a
preliminary draft of this manuscript.

\bibliographystyle{splncs03}
\bibliography{references_used}
\end{document}